%% file: Turing.tex
\documentclass{elsarticle}%
\usepackage{amssymb}
\usepackage{graphicx}
\usepackage{tcolorbox}
\usepackage{amsfonts}
\usepackage{amsmath}
\usepackage{slashbox}
\usepackage[a4paper,inner=1.2in,outer=1.2in,top=1.4in,bottom=1.5in]{geometry}%
\providecommand{\U}[1]{\protect\rule{.1in}{.1in}}
\newtheorem{theorem}{Theorem}

\newenvironment{proof}[1][Proof]{\noindent\textbf{#1.} }{\ \rule{0.5em}{0.5em}}
\begin{document}
\journal{arXiv} \input{tit}


\section{Introduction}

The idea that computers might be able to do some intellectual work has existed
for a long time and, as demonstrated by recent AI developments, is not merely
science fiction. AI has become successful in a range of applications,
including recognising images and driving cars. Playing human games such as
chess and Go has long been considered to be a major benchmark of human
capabilities. Computer programs have become robust chess players and, since
the late 1990s, have been able to beat even the best human chess champions;
though, for a long time, computers were unable to beat expert Go players ---
the game of Go has proven to be especially difficult for computers.

While most of the best early AI game-playing agents were specialised to play a
particular game (which is less interesting from the general AI perspective),
more recent game-playing agents often involve general machine learning
capabilities, and sometimes evolutionary algorithms
\cite{AI2011,ChessEvol2013}. Since 2005, general game playing was developed to
reflect the ability of AI to play generic games with arbitrary given rules
\shortcite{GGP2005}. Conceptually, the game model associated with the general
game playing resembles the schematisation of algorithms introduced by Turing
machines. In 2016, a new program called \textit{AlphaGo} finally won a victory
over a human Go champion, only to be beaten by its subsequent versions
(\textit{AlphaGo Zero} and\ \textit{AlphaZero).} \textit{AlphaZero} proceeded
to beat the best computers and humans in chess, shogi and Go, including all
its predecessors from the Alpha family
\shortcite{AlphaZeroScience2018}\footnote{While relative strength of these
programs has not been tested after \textit{Stockfish} upgrades
(\textit{AlphaZero} is not publicly available), the open-source chess engine
\textit{Stockfish}, which incorporates elements of machine learning since
version 12 of 2020, has the highest historical ELO ratings at present
\cite{CCRL2021}.}. Core to AlphaZero's success is its use of a deep neural
network, trained through reinforcement learning, as a powerful heuristic to
guide a tree search algorithm (specifically Monte Carlo Tree Search).

The recent successes of AI in playing games are good reason to consider the
limitations of learning algorithms and, in a broader sense, the limitations of
AI. In the context of a particular competition (or \textquotedblleft game"), a
natural question to ask is whether a comprehensive winner AI might exist ---
one that, given sufficient resources, will \emph{always} achieve the best
possible outcome. Following the typical modus operandi of reinforcement
learning, such as implemented in AlphaZero,\ we distinguish \textquotedblleft
learning algorithms" that require very large computational resources, and
\textquotedblleft playing programs" that have relatively modest computational
needs --- learning algorithms and playing programs perform different roles and
operate under different conditions. In this context, we focus on conceptual
limitations of general learning algorithms (rather than the actual game
dynamics), which are rooted in the availability of information about one's opponents.

\section{Examples of transitive and intransitive games}

The outcomes that can or cannot be achieved in a particular game depends on
the type of the game and the scope of information that is known about the
strategies used by the opponent. It appears that, for the purposes of our
analysis, numerous properties of various games can be reduced into two types
of games. In this section, we illustrate these two types by basic examples
shown in Figure \ref{fig1}.\ These examples are simple symmetric adversarial
games, where two opponents can use the same strategies against each other.

The \textquotedblleft dice game\textquotedblright\ shown in Figure
\ref{fig1}(a) is fully \textit{transitive} --- its strategies form a
transitive sequence $\boxed{1}\prec\boxed{2}\prec\boxed{3}\prec\boxed{4}\prec
\boxed{5}\prec\boxed{6}$, where \textquotedblleft$\prec$\textquotedblright%
\ indicates a winning strategy and $\boxed{a}\prec\boxed{b}\prec\boxed{c}$
always implies $\boxed{a}\prec\boxed{c}$ when the binary relation
\textquotedblleft$\prec$\textquotedblright\ is \textit{transitive}. It is
obvious that there exists a best possible strategy, $\boxed{6}$, which wins or
at least ensures a draw. The rock-paper-scissors game shown in Figure
\ref{fig1}(b) is strongly \textit{intransitive} since \textquotedblleft
rock\textquotedblright, \textquotedblleft paper\textquotedblright\ and
\textquotedblleft scissors\textquotedblright\ form an intransitive triplet
$\boxed{\operatorname{R}}\prec\boxed{\operatorname{P}}\prec
\boxed{\operatorname{S}}\prec\boxed{\operatorname{R}}$
\cite{Qcat2009M,Podd2013,K_Ent2015}. Obviously, there is no simple winning
strategy in this game. The outcome of this game, however, strongly depends on
whether a player is aware of the strategy used by their opponent: if player 1
knows the strategy of player 2, the first player can win easily.

\input{fig1}

Although the dice game is not significantly affected by information about the
opponent (strategy $\boxed{6}$ is always the best in Figure \ref{fig1}(a)),
knowing more about the opponent might still be practically useful. Indeed, a
player needs to examine $6^{2}=36$ possible outcomes but, if the strategy of
the opponent is known, can reduce this number to $6$. Even if the exact
opposing strategy is not known, the experience gained by playing against
opponents is still useful: we may be able to focus our exploration on strong
strategies (e.g. $\boxed{4}$, $\boxed{5}$ and $\boxed{6}$ in the dice game).
If the game is transitive but complicated, this can reduce the time and effort
needed for strategy optimisation. If the game is intransitive, the advantage
given by information about the opponent becomes crucial for success. Presence
of intransitivity in competitive environments is a common source of systemic
and evolutionary complexity \cite{K-PT2013}, and this tends to increase the
importance of information about one's competitors. In general, more realistic
games can involve various mixtures of transitive and intransitive rules. In
this case, strategies that are a priori deemed to be transitively superior can
be unexpectedly beaten by seemingly weak but intransitively effective alternatives.

As noted above, games can be played by computer programs at the levels that
routinely exceed human abilities. Intransitivity is common in competitions
among programs. According to CCRL website \cite{CCRL2021}, the result matrix
of the 12 leading chess engines is intransitive. For example, the engines
ranked 1, 2 and 6 in 2021 (the 4040 List) form an intransitive triplet:
\boxed{\textit{1.Stockfish}} $\succ$ \boxed{\textit{ 2.Fat Fritz }} $\succ$
\boxed{\textit{ 6.Houdini
}} $\succ$ \boxed{\textit{ 1.Stockfish}}. While the chess rules allow, at
least in principle, for transitive dominance of the best algorithms,
predicting winners can be even more difficult when game rules are explicitly
intransitive. The idea that difficulties of predicting winners in intransitive
competitions is related to algorithmic insolvability of the halting problem
was first suggested by Poddiakov \cite{Podd2006,Podd2013}. The present work
proves that, in the context of learning algorithms, insolvability of the
halting problem is indeed related to intransitivity and this imposes
conceptual limitations on the extent of learning that can be mastered by
algorithms in competitive environments. 

Considering learning algorithms playing a game against each other, we note
that the character of the game and information available about the opposing
algorithm are major factors affecting the outcome of the game. These factors
are taken into account in more accurate definitions of competitions between
algorithms, which are introduced further in the paper.

\section{Learning algorithms and playing programs}

Let AI be represented by a \textit{learning algorithm} $L,$ which is run on
computers that are practically unconstrained in terms of resources and speed,
to produce $L\longrightarrow p$ a \textit{playing program} $p,$ which is then
run in real time with limited resources. Due to these limitations, one can
assume that there is a large but finite number of possible programs.\ The
learning algorithm is understood in a general way. For example, algorithm $L$
can simply print a certain program $p,$ or $L$ may perform extensive
calculations to optimise its output. At this stage we consider only algorithms
that produce output and halt.

Adversarial competition between $p_{i}^{\prime}$ and $p_{j}^{\prime\prime}$
(which are selected from a finite set of allowed programs $p_{1}^{\prime
},...,p_{n^{\prime}}^{\prime}$\ and $p_{1}^{\prime\prime},...,p_{n^{\prime
\prime}}^{\prime\prime}$) is evaluated for a particular game by computable
algorithm $C$.\ The simplest way to encode a program $p_{i}$ is by its number
$i$, although $p_{i}$ can also be a program in any computer language that can
be executed by $C$. The rules of the game are reflected by $C$, and deemed to
be known to the learning algorithms, i.e. $L^{\prime}=L^{\prime}%
[C]\longrightarrow p_{i}^{\prime}$ and $L^{\prime\prime}=L^{\prime\prime
}[C]\longrightarrow p_{j}^{\prime\prime}$. For our purposes, it is sufficient
to consider zero-sum win/draw/loss two-player games, i.e. in which $C$ has the
form
\begin{equation}
C(p_{i}^{\prime},p_{j}^{\prime\prime})=\left\{
\begin{array}
[c]{ll}%
+1, & p_{i}^{\prime}\succ p_{j}^{\prime\prime},\ \ \text{(win for player 1)}\\
0, & p_{i}^{\prime}\sim p_{j}^{\prime\prime},\ \ \text{(draw)}\\
-1, & p_{i}^{\prime}\prec p_{j}^{\prime\prime},\ \ \text{(loss for player 1)}%
\end{array}
\right.  \label{Tur1}%
\end{equation}
As long as the resource allocations are sufficient, any finite game of this
category can be represented by such an algorithmic function $C(.,.)$; thus we
can consider $C(.,.)$ to itself be a definition of a game. The specific form
of the game rules --- with perfect or imperfect information, with simultaneous
or sequential moves, etc. --- does not affect our interpretation. As
schematically illustrated in Figure \ref{fig1}\ and accurately defined below,
we only need to distinguish two types of games (although these types are, of
course, affected by the game rules: for example, imperfect information tends
to stimulate intransitivity --- see ref.\citeA{K_Ent2015}).\ The algorithm $C$
allows only limited time (or a limited number of steps) for execution of the
programs. For example, if a program fails to make a move within given time
limits, it might immediately be judged as the loser; the time limits, however,
do not apply to the learning stage. Hence, each program is allowed to execute
only a finite number of steps (determined by the time limit) and the game has
only a finite number of states; resource limitations imply that the size of
the programs must be finite (and, perhaps, limited by the game rules);
therefore, the numbers of possible programs $n^{\prime}$ and $n^{\prime\prime
}$ are large but finite.

Irrespective of the specific nature of the game, the competition can be fully
specified by a finite $n^{\prime}\times n^{\prime\prime}$ table $C_{ij}%
=C(p_{i}^{\prime},p_{j}^{\prime\prime}),$ which is computable for the learning
algorithms. It is easy to see that each program $p_{i}^{\prime}$ can be
interpreted as strategy number $i$ out of the set of $n^{\prime}$ possible
game strategies, while $p_{j}^{\prime\prime}$ is just the strategy number $j$
out of the set of $n^{\prime\prime}$ strategies. Indeed the table $C_{ij}$ is
simply a \textit{normal-form} representation of the game.\ A\textit{\ Turing
machine} can easily compute $C(p_{i}^{\prime},p_{j}^{\prime\prime}),$ for
example, by selecting $C_{ij}$ from the game table. In the special case of a
symmetric game, the two sets of strategies are the same $n^{\prime}%
=n^{\prime\prime}=n$, $p_{i}^{\prime}=p_{i}^{\prime\prime}=p_{i},$ while the
payoff matrix must be antisymmetric: $C_{ij}=-C_{ji}.$ A symmetric game,
obviously, requires that $C_{ii}=C(p_{i},p_{i})=0$, but in asymmetric games
$C(p_{i},p_{i})\neq0$ is possible (assuming that strategy $p_{i}$ is available
to the first and to the second player). For example, the outcome
$C(p_{i},p_{i})=1$ might be common in a game with a first-mover advantage
(though the order of players in $C(\cdot,\cdot)$ is down to convention).

Since each program $p_{i}$ can be interpreted as a game strategy, the outcomes
of competition between these pure strategies can be defined in terms of the
classical game theory for zero-sum games with pure strategies. This game may
or may not have a Nash equilibrium in pure strategies. For our purposes, it is
sufficient to consider two alternatives: A) a \textit{game with transitive
domination}, which must have a pure strategy Nash equilibrium and B) a
\textit{strongly intransitive game}, which cannot. In the case with
\textit{transitive domination}, there is a strategy $p_{d}^{\prime}$ that is
transitively superior over all opposing strategies: $p_{d}^{\prime}\succ
p_{j}^{\prime\prime}$ for all $j$ (without loss of generality, we assume that
player 1 has this strategy). Note that in a symmetric game, strict transitive
domination $p_{i}^{\prime}\succ p_{j}^{\prime\prime}$ is not possible since
$p_{i}^{\prime}=p_{d}\sim p_{d}=p_{j}^{\prime\prime}.$ The game is
\textit{strongly intransitive} if each strategy $p_{i}^{\prime}$ has at least
one strategy $p_{j}^{\prime\prime}=S^{\prime\prime}(p_{i}^{\prime})$
(dependent on $p_{i}^{\prime}$) such that $p_{i}^{\prime}\prec p_{j}%
^{\prime\prime}$, and each strategy $p_{j}^{\prime\prime}$ has at least one
strategy $p_{\tilde{\imath}}^{\prime}=S^{\prime}(p_{j}^{\prime\prime})$
(dependent on $p_{j}^{\prime\prime})$ such that $p_{\tilde{\imath}}^{\prime
}\succ p_{j}^{\prime\prime}$. The functions $S^{\prime}(...)$ and
$S^{\prime\prime}(...)$ are \textquotedblleft best response\textquotedblright%
\ functions, which pick out (potentially out of several winning options) a
player's best answer to a fixed strategy from their opponent. In general
$S^{\prime}(...)$ and $S^{\prime\prime}(...)$ are different functions, but one
can choose $S^{\prime}(...)=S^{\prime\prime}(...)$ \ in a symmetric game.

It is obvious that a learning algorithm can guarantee victory as player 1 (or
at least a draw if the game is symmetric) in the transitive case (A), as long
as it finds the transitively superior strategy $L^{\prime}\rightarrow
p_{d}^{\prime}$, but any learning algorithm can be defeated in the
intransitive case (B): $L^{\prime}\rightarrow p_{i}^{\prime}$ is defeated by
$L^{\prime\prime}\rightarrow S^{\prime\prime}(p_{i}^{\prime}),$ which is in
turn defeated by $L^{\prime}\rightarrow S^{\prime}(S^{\prime\prime}%
(p_{i}^{\prime})),$ which is defeated by $L^{\prime\prime}\rightarrow
S^{\prime\prime}(S^{\prime}(S^{\prime\prime}(p_{i}^{\prime}))),$ etc. Of
course, if one were to allow random mixing of strategies, an algorithm
supplemented by a (quantum) random generator might be able to eke out some
statistical edge by playing a mixed-strategy Nash equilibrium, but this is far
from winning every game that in principle can be won.

One can also consider the case of playing multiple games, when a program
$p_{i}^{\prime}$ \ ($i=1,...,n^{\prime}$) competes against $p_{j}%
^{\prime\prime}$ \ ($j=1,...,n^{\prime\prime}$) in the games $C_{ij}^{_{(k)}}$
from a given set $C_{ij}^{_{(1)}},...,C_{ij}^{_{(m)}}$ of $m$ games with
$n^{\prime}\times n^{\prime\prime}$ possible payoffs. The outcome of the
series (win, draw or loss) is given by an $n^{\prime}\times n^{\prime\prime}$
overall payoff matrix $C_{ij}^{_{G}}$ that depends on the outcomes of each
game $C_{ij}^{_{(1)}},...,C_{ij}^{_{(m)}}$ as defined by the conditions of the
series. Note that the number $m_{n}$ of different games (\ref{Tur1}) with
$n^{\prime}\times n^{\prime\prime}$ payoffs is finite and cannot exceed
$m_{n}\leq3^{n^{\prime}\times n^{\prime\prime}},$ therefore $C_{ij}^{_{G}}$
must be one of these $m_{n}$ games. Although playing multiple games can make
the structure of the effective payoff matrix $C_{ij}^{_{G}}$ more complicated
and therefore increase practical difficulties experienced by the competitors,
this does not impose any new principal constraints on our analysis.

\section{Learning algorithms in open-source competitions}

There is, however, a more interesting case, in which the learning algorithms
can receive information about the strategies played by their opponents. The
case when the first algorithm $L^{\prime}$ knows the opposing program
$p_{j}^{\prime\prime}$ (that is, $L^{\prime}[p_{j}^{\prime\prime}]\rightarrow
p_{i}^{\prime}(p_{j}^{\prime\prime})$ but $L^{\prime\prime}\rightarrow
p_{j}^{\prime\prime}$ does not depend on $p_{i}^{\prime})$ gives a very
significant advantage to $L^{\prime}$ that, obviously, can be translated into
optimal strategies $p_{i}^{\prime}(p_{j}^{\prime\prime})$, even for
intransitive games. We are interested in the more complicated case in which
the learning algorithms are placed in symmetric or comparable conditions.
Under these conditions, the two opposing learning algorithms $L^{\prime}$ and
$L^{\prime\prime}$ receive each other's source code and then train two
opposing programs $p_{i}^{\prime}$ and $p_{j}^{\prime\prime}$ that compete in
real time. This implies that $L^{\prime}[L^{\prime\prime}]\rightarrow
p_{i}^{\prime}(L^{\prime\prime})$; that is, $L^{\prime}$ has $L^{\prime\prime
}$\ as an input to produce $p_{i}^{\prime}$ depending on $L^{\prime\prime},$
and $L^{\prime\prime}[L^{\prime}]\rightarrow p_{j}^{\prime\prime}(L^{\prime}%
)$, i.e. $L^{\prime\prime}$ has $L^{\prime}$\ as an input to produce
$p_{j}^{\prime\prime}$ depending on $L^{\prime}$. This information exchange is
fully symmetric, although the game, which is determined by $C(p_{i}^{\prime
},p_{j}^{\prime\prime})$, may be symmetric or asymmetric.\ Note that it is
generally unknown whether a particular algorithm can halt and produce an
output for a particular input. These learning algorithms that compete in
\textit{open-source competitions} can be formally specified using a
\textit{universal Turing machine}, $U\left\langle M\right\rangle
[D_{0}]\rightarrow D_{1}.$ Here, the universal Turing machine $U$ has header
$M$, is applied to input $D_{0}$ and produces output $D_{1}$. The learning
algorithms associated with open-source competitions are given by
$U\left\langle L^{\prime}\right\rangle [C,L^{\prime\prime}]\rightarrow
p_{i}^{\prime}$ and $U\left\langle L^{\prime\prime}\right\rangle [C,L^{\prime
}]\rightarrow p_{j}^{\prime\prime},$ where the computable algorithmic function
$C=C(p_{i}^{\prime},p_{j}^{\prime\prime})$ (or the corresponding game payoff
table $C_{ij}$) is emphasised to be available to both competing Turing
machines. The two algorithms implemented by the universal Turing machine $U$
represent two adversarial goals and cannot be reduced to a single algorithm.
Therefore it is possible to consider learning algorithms that are expected to
play any game with number of payoffs not exceeding $n^{\prime}\times
n^{\prime\prime}$.

An algorithm $L^{\prime}$ can win over algorithm $L^{\prime\prime}$ by either
producing a winning program $L^{\prime}[L^{\prime\prime}]\rightarrow
p_{i}^{\prime}(L^{\prime\prime})\succ p_{j}^{\prime\prime}(L^{\prime})$, where
$L^{\prime\prime}[L^{\prime}]\rightarrow p_{j}^{\prime\prime}(L^{\prime})$, or
by producing an output $L^{\prime}[L^{\prime\prime}]\rightarrow p_{i}^{\prime
}(L^{\prime\prime})$ and demonstrating that the competing algorithm does not
halt $L^{\prime\prime}[L^{\prime}]\rightarrow\varnothing^{\prime\prime}$ (at
this point our analysis assumes existence of an agreed axiomatic system that
determines correctness of demonstrations). Therefore, at least some of the
algorithms that fail to halt (e.g. a computer program with an infinite loop)
can be defeated. Note that the outcome of the competition may remain undecided
(e.g. when both algorithms run indefinitely).

Can a learning algorithm, say $L^{\prime},$ defeat all opposing algorithms
when information about the algorithms is exchanged? In the case of an
asymmetric game with \textit{transitive domination }(A), any opponent that
halts and produces an output $p_{j}^{\prime\prime}$ is defeated by:
$L^{\prime}[L^{\prime\prime}]\rightarrow p_{d}^{\prime}\succ p_{j}%
^{\prime\prime}$ --- the winning algorithm $L^{\prime}$ simply ignores its
input $L^{\prime\prime}$ and selects $p_{i}^{\prime}=p_{d}^{\prime}$. The case
of \textit{strongly intransitive} competition (B) is, by contrast, quite a bit
more complicated. In this case, any strategy $p_{i}^{\prime}$ that $L^{\prime
}$ produces could (at least in principle) be beaten by an opponent that
outputs $S^{\prime\prime}(p_{i}^{\prime})$; thus an analogous algorithm that
ignores its input seems especially foolhardy.

More specifically, we consider an algorithm $L^{\prime}$ to be a
\textit{universal winner} for a given game $C(...)$ when it can defeat every
opposing algorithm $L^{\prime\prime}$. At this point we can formulate the
following theorem

\begin{theorem}
\label{Lem1} Any algorithm competing in an open-source competition associated
any strongly intransitive game cannot be a universal winner.
\end{theorem}

\begin{proof}
Indeed, let algorithm $L^{\prime}=L_{w}$ be such a universal winning algorithm
that can defeat any opposing algorithm $L^{\prime\prime}$. The winning
algorithm $L_{w}[L^{\prime\prime}]$ must always halt for any $L^{\prime\prime
},$ and must produce an output $L_{w}[L^{\prime\prime}]\rightarrow
p_{i}^{\prime}(L^{\prime\prime})$ dependent on $L^{\prime\prime}$. Consider
the following $L^{\prime\prime}$: it runs $L_{w}[L^{\prime\prime}] $ to obtain
its output, $p_{i}^{\prime}=p_{i}^{\prime}(L^{\prime\prime}),$ and then
selects $p_{j}^{\prime\prime}=S^{\prime\prime}(p_{i}^{\prime})$, which defeats
$L_{w}$. Hence $L_{w}$ is defeated by at least one algorithm $L^{\prime\prime
}$ and cannot be a universal winner.
\end{proof}

\section{Relation to the halting problem}

Assuming that sufficient information is available to the opposition, even the
best learning algorithms cannot be universal winners for all winnable games,
since at least some of these games must be strongly intransitive and are
subject to theorem \ref{Lem1} (note that strongly intransitive games are, by
definition, winnable). One of the most interesting features of this statement
is that it is closely related to the famous incompleteness theorems of
G\"odel\citeA{Godel1931} and Turing\citeA{Turing1937}. The results of Turing
were later generalised in the Rice theorem\citeA{Rice1953}. The incompleteness
theorems have implications for the existence of a halting function
\begin{equation}
H(L,D)=H(L[D])=\left\{
\begin{array}
[c]{l}%
1,\ \ \text{when}\ L[D]\text{ halts }\\
0,\ \ \text{when }L[D]\text{ does not halt }%
\end{array}
\right.  \label{Tur0}%
\end{equation}
that uses the agreed axiomatic system to determine whether algorithm (program)
$L$ halts or runs forever when applied to input data $D$. We can show that the
impossibility of an algorithm implementing the universal halting function
follows from our considerations.

\begin{theorem}
\label{Lem2}(Turing) A universal computable halting function $H(L,D)=H(L[D])$
does not exist --- it cannot be computable for all algorithms $L$ and all data
sets $D$.
\end{theorem}

\begin{proof}
If the halting function $H(...)$ is universally computable, one can easily
construct a universal winning algorithm $L^{\prime}=L_{w}$:
\begin{equation}
L_{w}[L^{\prime\prime}]:\left\{
\begin{array}
[c]{l}%
H(L^{\prime\prime},L_{w})=0,\ \ \rightarrow p_{1}^{\prime}\\
H(L^{\prime\prime},L_{w})=1,\ \ \rightarrow S^{\prime}(p_{j}^{\prime\prime
}(L_{w}))
\end{array}
\right.  \label{Tur2}%
\end{equation}
for a selected strongly intransitive game $C(...)$. That is, if $L_{w}%
[L^{\prime\prime}]$ determines that $L^{\prime\prime}[L_{w}]$ runs forever,
then $L_{w}$ can print step-by-step execution of $H(...)$\ to demonstrate
$H(L^{\prime\prime},L_{w})$ $=$ $0,$ produce any output, say $L_{w}%
[L^{\prime\prime}]\rightarrow p_{1}^{\prime},$ halt and declare its victory.
If $L_{w}[L^{\prime\prime}]$ determines that $L^{\prime\prime}[L_{w}]$ halts
and $L^{\prime\prime}[L_{w}]\rightarrow p_{j}^{\prime\prime}(L_{w})$, then
$S^{\prime}(p_{j}^{\prime\prime}(L_{w}))$ is selected. Since a universal
winning algorithm cannot exist according to theorem \ref{Lem1}, this also
prohibits the existence of a universal halting function and proves theorem
\ref{Lem2}.
\end{proof}

The halting problem is also related to the incompleteness theorem of
G\"{o}del, which, in the context of our consideration, points to the existence
of correct but unprovable statements (e.g. the statement \textquotedblleft%
$L[D]$ runs forever\textquotedblright). It appears that incompleteness of
formal systems is promoted by intransitivity, at least under conditions of
competitive environments considered in theorems \ref{Lem1} and \ref{Lem2}: in
simple terms, intransitivity makes our knowledge incomplete and relativistic.

\section{Failure to halt as a game strategy}

In the context of the competitions considered here, the incompleteness
theorems allow us to broaden the statement of theorem \ref{Lem1}:

\begin{theorem}
\label{Lem3} Any algorithm competing in an open-source competition cannot be a
universal winner.
\end{theorem}

\begin{proof}
Let us assume that algorithm $L^{\prime}=L_{w}$ is such a universal winning
algorithm. Due to theorem \ref{Lem2}, this algorithm cannot implement a
universal halting function and there must exist an algorithm $L_{h}$, whose
halting status remains unknown when applied to some data $D_{h}$ (note that
$L_{h}$ and $D_{h}$ may or may not have some direct relevance to the current
game). Hence, the algorithm $L^{\prime\prime}[L_{w}]$ that ignores $L_{w}$ and
executes $L_{h}[D_{h}]$ is not defeated by $L_{w}[L^{\prime\prime}]$. This
leads us to a contradiction.
\end{proof}

Note that the universal winner $L_{w}$ is shown to be defeated in the
intransitive conditions of theorem \ref{Lem1} by a known, specific strategy of
the opposition, $L^{\prime\prime}[L_{w}]\rightarrow p_{j}^{\prime\prime
}=S^{\prime\prime}(p_{i}^{\prime})$; while $L_{w}$ only fails to defeat
$L^{\prime\prime}$\ in theorem \ref{Lem3}. In the latter case, $L^{\prime
\prime}$ executes the algorithm $L_{h}[D_{h}],$ which is only known to exist
but not explicitly specified or nominated. Generally, both $L^{\prime\prime}$
and $L^{\prime}$ might not know whether a specific $L_{h}[D_{h}]$ halts or not
since, if the fact that $L_{h}[D_{h}]$ does not halt\ is known to
$L^{\prime\prime}$ using $L_{h},$ then $L^{\prime}$ might be able to
demonstrate this and win. Hence, competitors may have to use algorithms
without complete knowledge of their performance. The abstract strategies used
to prove theorem \ref{Lem3} may have more relevance to the real world than one
might think. For example, the weaker side $L^{\prime\prime}$ may realise that
it would lose to $p_{i}^{\prime}=p_{d}^{\prime}$ with any strategy
$p_{j}^{\prime\prime}$ it can muster, and simply refuse to halt $L^{\prime
\prime}[L^{\prime}]\rightarrow\varnothing^{\prime\prime},$ while hiding its
intentions from the opposition so that $L^{\prime}$ cannot prove that
$H(L^{\prime\prime},...)=0$. Does this, perhaps, resemble the situation in the
last US presidential elections?

One can see that the logic of the adversarial game converts a failure to solve
the problem and halt into a new game strategy, $\varnothing^{\prime\prime},$
assuming that $H(L^{\prime\prime},...)$ is not computable. Even if
$p_{b}^{\prime}$ is transitively dominant, the strategy $\varnothing
^{\prime\prime}$ introduces some intransitivity into the game since
$p_{b}^{\prime}\succ p_{j}^{\prime\prime}\succ p_{i}^{\prime}$ (with some
suitable selection of $j$ and $i$), but $p_{b}^{\prime}\sim\varnothing
^{\prime\prime}\sim p_{i}^{\prime}$, although this intransitivity is weaker
than that in strongly intransitive games.\ Intransitivity imposes limitations
on universal winning strategies, which, as determined by theorems \ref{Lem1}
and \ref{Lem2}, are associated with the impossibility of universal computable
halting functions.

\section{Discussion and conclusion}

Considering the capacity of learning algorithms to learn and adapt to succeed
in competitions, we note the importance of information about one's
competitors. One-sided availability of such information makes competition
highly uneven. It is not a surprise that, given both substantial computational
resources and full information about a particular opponent, a good learning
algorithm should be able to defeat this opponent (assuming, of course, that
this is permitted by the game rules). We demonstrate, however, that if
information is exchanged pari passu with the opposition, one's ability to
manage the outcomes of adversarial games is necessarily limited by algorithmic
incompleteness, and no learning algorithm can become a universal winner in
complex competitions, especially when the relevant competition rules are
intransitive. In the context of adversarial games, we do not presume
algorithmic incompleteness by invoking the incompleteness theorems, but
demonstrate incompleteness in game conditions by using the potential
intransitivity of competitive environments. Our consideration involves a
standard, Turing-like interpretation of computer algorithms but in conditions
when the algorithms are subordinated to the presence of conflicting goals
pursued by the competitors.

A basic learning algorithm to evolve a chess engine towards better performance
can be created without major difficulties, but achieving successful learning
under intransitive arrangements appears to be much more problematic
\citeA{ChessEvol2013}. A major question related to AI is whether AI can
perform well in more complex and uncertain situations, especially when
significant intransitivity is present in a competitive environment. Our
analysis illustrates that mere existence of intransitivity under these
conditions is sufficient for demonstrating incompleteness of our knowledge. We
show that no algorithm can become undefeatable and exercise full control over
an intransitive competitive environment by subjugating all other agents to its
own goals, unless this algorithm is advantageously benefited by asymmetric
availability of information or resources. Good learning algorithms should be
capable of converting knowledge about the opposing algorithms into a
significant advantage. The impossibility of controlling competitive
environments in intransitive conditions points to emergence of complexity,
although the question of whether AI systems can reach the higher levels of
complexity associated with known complex evolutionary systems (e.g. that of
biological, social and technological systems, and of human intelligence and
organisation) remains open.



\vskip0.2in
\bibliographystyle{elsarticle-num-names}
\bibliography{comp}

\end{document}

%% file: tit.tex
\newcommand{\shortcite}{\cite}
\newcommand{\citeA} {\cite}

\begin{frontmatter}
\title{On limitations of learning algorithms in competitive environments}
\author{A.Y. Klimenko and D.A. Klimenko \\
SoMME, The University of Queensland, QLD 4072, Australia}
\begin{abstract}

We discuss conceptual limitations of generic learning algorithms pursuing adversarial goals in competitive environments, 
and prove that they are subject to limitations 
that are analogous to the constraints on knowledge imposed by the famous
theorems of G\"{o}del and Turing. These limitations are shown to be related to intransitivity, which is 
commonly present in competitive environments. 

\end{abstract}

\begin{keyword}

{learning algorithms \sep competitions \sep incompleteness theorems}

\end{keyword}
\end{frontmatter}

%% file: fig1.tex
\begin{figure}[h!]
\begin{center}
\includegraphics[width=15cm,page=1,trim=0cm 6cm 0cm 6cm, clip ]{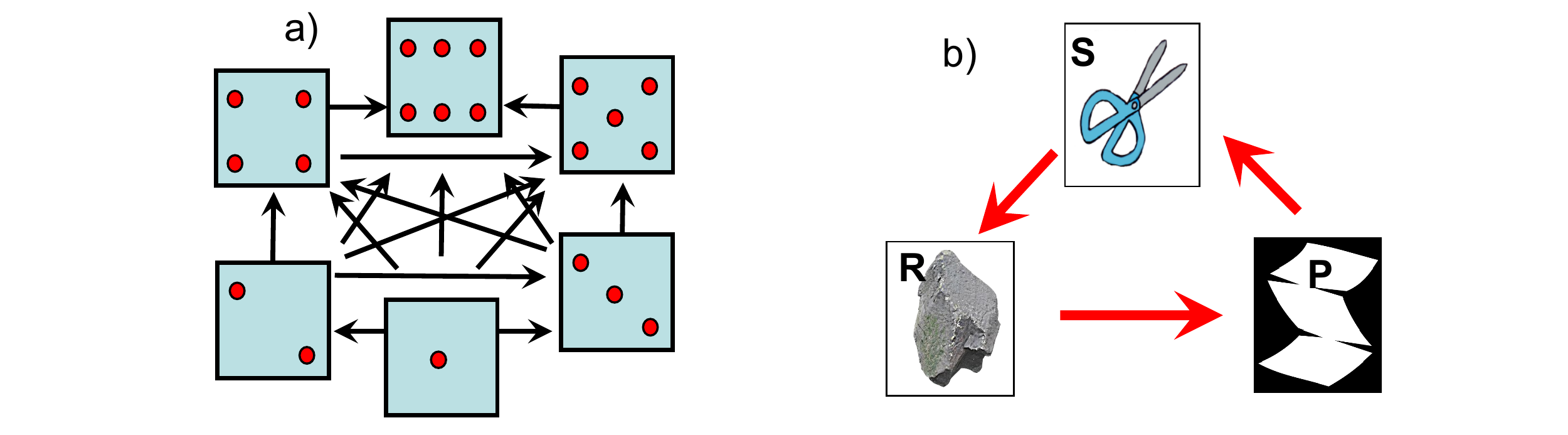}
\caption{ Examples of simple symmetric games: a) fully transitive: dice game, b) strongly intransitive: rock-paper-scissors game. The arrows point from losing to winning strategies.   }
\label{fig1}
\end{center}
\end{figure}

%% file: Turing.bbl
\begin{thebibliography}{13}
\providecommand{\natexlab}[1]{#1}
\providecommand{\url}[1]{\texttt{#1}}
\providecommand{\urlprefix}{URL }
\expandafter\ifx\csname urlstyle\endcsname\relax
  \providecommand{\doi}[1]{doi:\discretionary{}{}{}#1}\else
  \providecommand{\doi}[1]{doi:\discretionary{}{}{}\begingroup
  \urlstyle{rm}\url{#1}\endgroup}\fi
\providecommand{\bibinfo}[2]{#2}

\bibitem[{Bach and Edelkamp(2011)}]{AI2011}
\bibinfo{author}{J.~Bach}, \bibinfo{author}{S.~Edelkamp}, \bibinfo{title}{KI
  2011: Advances in Artificial Intelligence}, Lecture Notes in Computer
  Science, \bibinfo{publisher}{Springer}, \bibinfo{address}{Berlin,
  Heidelberg}, \bibinfo{year}{2011}.

\bibitem[{Schreiber and Bramstang(2013)}]{ChessEvol2013}
\bibinfo{author}{J.~Schreiber}, \bibinfo{author}{P.~Bramstang},
  \bibinfo{title}{Evolutionary Tuning of Chess Playing Software},
  \bibinfo{howpublished}{Degree Project, KTH Royal Institute of Technology},
  \bibinfo{year}{2013}.

\bibitem[{Genesereth et~al.(2005)Genesereth, Love, and Pell}]{GGP2005}
\bibinfo{author}{M.~Genesereth}, \bibinfo{author}{N.~Love},
  \bibinfo{author}{B.~Pell}, \bibinfo{title}{General game playing: overview of
  the AAAI competition}, \bibinfo{journal}{The AI magazine}
  \bibinfo{volume}{26}~(\bibinfo{number}{2}) (\bibinfo{year}{2005})
  \bibinfo{pages}{62}.

\bibitem[{Silver et~al.(2018)Silver, Hubert, Schrittwieser, Antonoglou, Lai,
  Guez, Lanctot, Sifre, Kumaran, Graepel, Lillicrap, Simonyan, and
  Hassabis}]{AlphaZeroScience2018}
\bibinfo{author}{D.~Silver}, \bibinfo{author}{T.~Hubert},
  \bibinfo{author}{J.~Schrittwieser}, \bibinfo{author}{I.~Antonoglou},
  \bibinfo{author}{M.~Lai}, \bibinfo{author}{A.~Guez},
  \bibinfo{author}{M.~Lanctot}, \bibinfo{author}{L.~Sifre},
  \bibinfo{author}{D.~Kumaran}, \bibinfo{author}{T.~Graepel},
  \bibinfo{author}{T.~Lillicrap}, \bibinfo{author}{K.~Simonyan},
  \bibinfo{author}{D.~Hassabis}, \bibinfo{title}{A general reinforcement
  learning algorithm that masters chess, shogi, and Go through self-play},
  \bibinfo{journal}{Science} \bibinfo{volume}{362}~(\bibinfo{number}{6419})
  (\bibinfo{year}{2018}) \bibinfo{pages}{1140--1144}.

\bibitem[{CCR(2021)}]{CCRL2021}
\bibinfo{title}{Computer Chess Rating Lists (CCRL)},
  \bibinfo{howpublished}{http://ccrl.chessdom.com/ccrl/4040},
  \bibinfo{year}{({J}une, 13, 2021)}.

\bibitem[{Makowski(2009)}]{Qcat2009M}
\bibinfo{author}{M.~Makowski}, \bibinfo{title}{Transitivity vs. intransitivity
  in decision making process --- an example in quantum game theory},
  \bibinfo{journal}{Physics Letters A}
  \bibinfo{volume}{373}~(\bibinfo{number}{25}) (\bibinfo{year}{2009})
  \bibinfo{pages}{2125--2130}.

\bibitem[{Poddiakov and Valsiner(2013)}]{Podd2013}
\bibinfo{author}{A.~Poddiakov}, \bibinfo{author}{J.~Valsiner},
  \bibinfo{title}{Intransitivity cycles and their transformations: How
  dynamically adapting systems function}, in: \bibinfo{editor}{L.~Rudolph}
  (Ed.), \bibinfo{booktitle}{Qualitative Mathematics for the Social Sciences:
  Mathematical Models for Research on Cultural Dynamics},
  \bibinfo{publisher}{Routledge}, \bibinfo{address}{Abingdon, NY},
  \bibinfo{pages}{343--391}, \bibinfo{year}{2013}.

\bibitem[{Klimenko(2015)}]{K_Ent2015}
\bibinfo{author}{A.~Y. Klimenko}, \bibinfo{title}{Intransitivity in Theory and
  in the Real World}, \bibinfo{journal}{Entropy}
  \bibinfo{volume}{17}~(\bibinfo{number}{6}) (\bibinfo{year}{2015})
  \bibinfo{pages}{4364--4412}.

\bibitem[{Klimenko(2013)}]{K-PT2013}
\bibinfo{author}{A.~Y. Klimenko}, \bibinfo{title}{Complex competitive systems
  and competitive thermodynamics}, \bibinfo{journal}{Phil. Trans. R. Soc. A}
  (\bibinfo{year}{2013}) \bibinfo{pages}{20120244}.

\bibitem[{Poddiakov(2006)}]{Podd2006}
\bibinfo{author}{A.~N. Poddiakov}, \bibinfo{title}{Intransitivity of
  superiority relations and decision-making}, \bibinfo{journal}{Psychology.
  Journal of the Higher School of Economics (in Russian)}
  \bibinfo{volume}{3}~(\bibinfo{number}{3}) (\bibinfo{year}{2006})
  \bibinfo{pages}{88--111}.

\bibitem[{Gödel(1931)}]{Godel1931}
\bibinfo{author}{K.~Gödel}, \bibinfo{title}{Über formal unentscheidbare
  Sätze der Principia Mathematica und verwandter Systeme, I},
  \bibinfo{journal}{Monatshefte für Mathematik und Physik}
  \bibinfo{volume}{38}~(\bibinfo{number}{1}) (\bibinfo{year}{1931})
  \bibinfo{pages}{173--198}.

\bibitem[{Turing(1937)}]{Turing1937}
\bibinfo{author}{A.~M. Turing}, \bibinfo{title}{On Computable Numbers, with an
  Application to the Entscheidungsproblem}, \bibinfo{journal}{Proceedings of
  the London Mathematical Society} \bibinfo{volume}{42}~(\bibinfo{number}{1})
  (\bibinfo{year}{1937}) \bibinfo{pages}{230--265}.

\bibitem[{Rice(1953)}]{Rice1953}
\bibinfo{author}{H.~G. Rice}, \bibinfo{title}{Classes of recursively enumerable
  sets and their decision problems}, \bibinfo{journal}{Transactions of the
  American Mathematical Society} \bibinfo{volume}{74}~(\bibinfo{number}{2})
  (\bibinfo{year}{1953}) \bibinfo{pages}{358--358}.

\end{thebibliography}
